\documentclass{article}

\usepackage{arxiv}
\usepackage[utf8]{inputenc} 
\usepackage[T1]{fontenc}    
\usepackage{hyperref}       
\usepackage{url}            
\usepackage{booktabs}       
\usepackage{amsfonts}       
\usepackage{nicefrac}       
\usepackage{microtype}      
\usepackage{lipsum}
\usepackage{graphicx}
\usepackage{subfigure}
\usepackage{booktabs} 
\usepackage{amsmath,amsfonts}
\usepackage{color}
\usepackage{hyperref}
\usepackage{amsthm}
\usepackage{verbatim}

\newtheorem{theorem}{Theorem}[section]
\newtheorem{corollary}{Corollary}[theorem]

\newcommand{\M}{\mathbf}
\title{On Coresets for Regularized Regression}

\author{
  Rachit Chhaya\\
  IIT Gandhinagar\\
  \texttt{rachit.chhaya@iitgn.ac.in} \\
   \And
 Anirban Dasgupta \\
  IIT Gandhinagar\\
  \texttt{anirbandg@iitgn.ac.in} \\
  \AND
   Supratim Shit \\
  IIT Gandhinagar\\
  \texttt{supratim.shit@iitgn.ac.in} \\
}

\begin{document}
\nocite{*}
\maketitle


\begin{abstract}
We study the effect of norm based regularization on the size of coresets for regression problems. Specifically, given a matrix $ \mathbf{A} \in {\mathbb{R}}^{n \times d}$ with $n\gg d$ and a vector $\mathbf{b} \in \mathbb{R} ^ n $ and $\lambda > 0$, we analyze the size of coresets for regularized versions of regression of the form $\|\mathbf{Ax}-\mathbf{b}\|_p^r + \lambda\|{\mathbf{x}}\|_q^s$ . Prior work has shown that for ridge regression (where $p,q,r,s=2$) we can obtain a coreset that is smaller than the coreset for the unregularized counterpart i.e. least squares regression~\cite{avron2017sharper}. We show that when $r \neq s$, no coreset for regularized regression can have size smaller than the optimal coreset of the unregularized version. The well known lasso problem falls under this category and hence does not allow a coreset smaller than the one for least squares regression. We propose a modified version of the lasso problem and obtain for it a coreset of size smaller than the least square regression. We empirically show that the modified version of lasso also induces sparsity in solution, similar to the original lasso. We also obtain smaller coresets for $\ell_p$ regression with $\ell_p$ regularization. We extend our methods to multi response regularized regression. Finally, we empirically demonstrate the coreset performance for the modified lasso and the $\ell_1$ regression with $\ell_1$  regularization.
\end{abstract}
\section{Introduction}
Most applications of machine learning require huge amounts of data to train models and hence computational efficiency is a cause of concern. A common strategy is to train the model on a judiciously selected subsample of the data. A coreset~\cite{har2004coresets,bachem2017practical} is a subsample of appropriately reweighted points from the original data which can be used to train models with competitive accuracy and provable guarantees. The size of a coreset is usually independent of the size of the original dataset making training on them much quicker.

Regression is a widely used technique in machine learning and statistics, the most popular variants being the least square regression ($\ell_2$ regression) and least absolute deviation ($\ell_1$ regression).  A coreset construction for $\ell_2$ regression based on leverage scores is given in \cite{drineas2006sampling} whereas coresets for $\ell_p$ regression have been created based either on norms of the so called well-conditioned basis~\cite{dasgupta2009sampling,sohler2011subspace} or based on Lewis weights~\cite{cohen2015p}.

A common variant of regression is to use regularization, meant to either achieve numerical stability, to prevent overfitting or to induce sparsity in the solution-- ridge and lasso being the most commonly used regularizers. Since regularization imposes a constraint on the solution space, we can potentially expect regularized problems to have a smaller size coreset, than the unregularized version. Indeed, this intuition has been formalized in the case of ridge regression~\cite{woodruff2014sketching} using the ridge leverage scores. Pilanci et al.~\cite{pilanci2015randomized} construct small sized coresets for constrained version of lasso using random projections; these coresets are, however, not (scaled) subsample of data points. To the best of our knowledge ours is the first work to study coresets for $\ell_p$ regularized regression for $p\neq 2$. Our first result is negative-- we show that 
it is not always possible to build smaller coresets for regularized regression. For a specific class of problems we show that smaller coresets are possible and we show how to construct them.

To construct such coresets, we follow the sensitivity framework given in \cite{feldman2011unified}. Sensitivities, defined in~\cite{langberg2010universal}, capture the importance of a data point for a particular optimization function. If we sample points using a probability distribution based on sensitivities, the sample size depends on the sum of the sensitivities and the dimension of the solution space. The core idea behind our bounds is that due to regularization, the sensitivities of points change -- while the sensitivity of very low sensitivity points might increase slightly (by additive $1/n$), the higher sensitivities are pulled down. The overall effect is that 
the sum of sensitivities for a regularized version of regression is less than the sum of sensitivities for its unregularized counterpart by a factor that depends on the value of the regularizer. 

\subsection{Our Contributions}
\begin{itemize}
	\item We first show that for any regularized problem of the form $\min _{\mathbf{x} \in \mathbb{R}^d} \|\mathbf{Ax}\|_p^r + \lambda \|\M{x}\|_q^s $, where $r \neq s$, a coreset for the problem also works as a coreset for its unregularized version. This implies that when $r \neq s$, we cannot build coresets with size smaller than the ones for the unregularized version.
	\item We introduce a modified version of lasso regression for which we show that we can construct coresets with size smaller than that of the unregularized linear regression.
	\item We calculate sensitivity upper bounds for the  $\ell_p $ regression with $\ell_p$ regularization. We focus on $p \geq 1, p\neq 2$. Specifically we give smaller coreset for the regularized least deviation problem i.e. for $p=1$.
	\item We show experimental evidence that the modified lasso problem also preserves sparsity like the lasso problem and hence is a suitable replacement. We also demonstrate the empirical performance of our sampling probabilities for the modified lasso problem and the regularized least deviation problem.
\end{itemize}
\subsection{Organization of the Paper}
The rest of the paper is organized in the following manner. In section 2 we discuss work in the areas related to coresets in general as well as coresets specifically for regularized problems and distinguish our work from existing work. In section 3 we provide all the notations, definitions and existing results that we use throughout the paper. In section 4 we give our main result relating coresets for regularized regression to the ones for unregularized regression. In section 5 we introduce the modified lasso  problem and analyze the size of coreset for it. In section 6 we show smaller coresets for $\ell_p$ regression with $\ell_p$ regularization. In section 7 we report the experiments validating our claims and  conclude in section 8 with discussion on future scope in this area.
\section{Related Work}
Coresets are small summaries of data which 
can be used as a proxy to the original data with provable guarantees. The term was first introduced in \cite{agarwal2004approximating} where they used coresets for the shape fitting problem. Coresets for clustering problem were described in~\cite{har2004coresets}. 
In \cite{feldman2011unified} authors gave a generalized framework to construct coresets based on importance sampling using sensitivity scores introduced in \cite{langberg2010universal} which was improved in \cite{braverman2016new} for both offline and streaming settings.

There is a large amount of work to reduce size of the data for the $\ell_2$ regression problem \cite{drineas2011faster,clarkson2017low}. Interested readers may refer to a good survey \cite{woodruff2014sketching} and references therein. Coresets for $\ell_2$ regression can be obtained using the popular leverage scores~\cite{drineas2006sampling}. Work has also been done to obtain these leverage scores in an efficient manner \cite{drineas2012fast}. Practical effectiveness of coresets for least square problems has been recently shown by \cite{maalouf2019fast} compared to available libraries like scikit-learn. Coresets for $\ell_p$ regression are obtained using the row norms of the well-conditioned basis~\cite{dasgupta2009sampling} or the Lewis weights~\cite{cohen2015p}. Other works that construct coresets for $\ell_p$ regression in general or $\ell_1$ regression in particular include \cite{sohler2011subspace,meng2013low,woodruff2013subspace,clarkson2016fast,dickens2018leveraging}. 

However not much has been done for regularized version of the regression problem. Pilanci et al.~\cite{pilanci2015randomized} are able to reduce the size of the popular lasso problem using random projections. However they work with the constrained version of the problem and not the regularized one. Also they do not obtain a sample of the original dataset. Reddi et al.~\cite{reddi2015communication} obtained additive error coresets for empirical loss minimization. Coresets using local sensitivity scores \cite{raj2019importance} also work for functions with a regularization term. Here they use a quadratic approximation of the function and then use leverage scores to approximate sensitivities locally. Tolochinsky et al. \cite{tolochinsky2018coresets} actually add a regularization term to obtain coresets for functions for which otherwise sublinear coresets may not exist.

Coresets for logistic regression and SVM with $\ell_2^2$ regularization were obtained by \cite{curtain2019coresets} using uniform sampling. 
They considered empirical loss minimization problems of the form $f(\M{w})=\sum_{i=1}^n l(-y_i{\M{w^T}\M{x_i})}+ \lambda r(R{\M{w}})$ where $\M{x_i}$ is the $i^{th}$ data point with corresponding label $y_i \in \{+1,-1\}$ and ${\M{w}}$ is solution vector. The parameter $R$ is maximum $2$-norm of a row in data matrix.
Their coresets do not work for a general response vector as in regression problems.
In addition, for uniform sampling to work, they assume a $(\sigma,\tau)$ condition on the loss and the regularization functions $l$ and $r$ respectively which says that $l(-\sigma) > 0$, and if $\|\M{w}\|_2 \geq \sigma$
then $r(\M{w}) \geq \tau l(\|\M{w}\|_2) $. If we consider response vectors consisting of $\{-1,+1\}$, we can formulate the regression problem to fit their loss function template in the following manner--- first create augmented $\M{x'}= \begin{bmatrix} \mathbf{x}\\ y_i \end{bmatrix}$ and $\M{w'}= \begin{bmatrix} \mathbf{w}\\ -1 \end{bmatrix} $. The lasso problem, under the restriction that $y \in \{-1, 1\}^d$, can be represented in their framework using  $\M{x'}$ and $\M{w'}$, setting $l(s) =s^2$ and $r(\M{w})= \|\M{w}\|_1$. Here $l(-\sigma) > 0$ for all values of $\sigma$. However, there is no $\tau$ that satisfies $\|\M{w}\|_1 \geq \tau\|\M{w}\|_2^2$ for all $\M{w}$ with $\|\M{w}\|_2 \geq \sigma$. So their framework does not apply to our case.

In spirit, our work most closely relates to the work in \cite{avron2017sharper}. Here the authors show that for ridge regression, coresets of size smaller than the coreset for the unregularized least square regression problem can be constructed. Their coreset size is a function of the statistical dimension of the matrix for some regularization parameter $\lambda > 0$. It is important to note that a coreset for unregularized regression can be shown to work for regularized regression as well. However the work in \cite{avron2017sharper} supports the intuition that coresets for regularized regression may be smaller than coresets for its unregularized counterpart. We generalize this idea from ridge regression to $\ell_p$ regression with $\ell_p$ regularization in this paper. We also show that for a broad class of regularized regression problems, it is not possible to construct coreset smaller in size compared to its unregularized form. Another recent related work on ridge regression is \cite{kachamoptimal}. In this paper they construct a deterministic coreset for ridge regression by using the BSS\cite{batson2012twice} technique and improve the dependence on $\epsilon$.
They also provide lower bounds showing this is tight.
Our lower bound is for the case $r \neq s$ in regularized regression, and is not applicable for ridge regression.
Our coresets are randomized, and are for regularized $\ell_p$ regression for all values of p. The BSS technique only works for p=2. Constructing deterministic coresets for regularized regression for all p is an interesting open question.
\section{Notations and Preliminaries}
In this section we describe all the notations, definitions and existing results that we will use throughout this paper. 
A scalar or a function (specified by context) is denoted by a lower case letter, e.g. $x$ while a vector is denoted by a boldface lower case letter, e.g. $\M{x}$. 
A matrix or a dataset, as defined by the context, is denoted by a boldface  uppercase letter, e.g. $\M{X}$.
A non bold face upper case letter also represents some scalar value unless stated otherwise.
We represent a dataset with a matrix where the rows of the matrix represent data points and the columns represent the features.
Vectors are considered column vectors unless stated otherwise.
$\M{a}_i$ represents the $i^{th}$ row of matrix $\M{A}$ while $\M{a}^k$ represents its $k^{th}$ column and $a_{ij}$ represents the entry in its $i^{th}$ row and $j^{th}$ column.
$O$ represent big-O in the ordered notation. For two quantities $x$ and $b$ and a suitable $\epsilon > 0$ we will often write $x \in (1 \pm \epsilon) b$ to compactly represent that $ (1 - \epsilon)b \le x \le (1 + \epsilon) b$. All statements where we say ``high probability", hold with probability at least some large constant, e.g. $0.99$, unless otherwise stated.

For a vector $\M{x} \in \mathbb{R}^d$, $\|\M{x}\|_p$, $p \geq 1$ represents its $p^{th}$ norm which is defined as $\|\M{x}\|_p= ({\sum_{j=1}^d|x_j|^p})^\frac{1}{p}$. 
The infinity norm for a vector  is given as $\|\M{x}\|_\infty = \max_j |x_j| $. 
For a matrix $\M{A} \in \mathbb{R}^{n \times d}$, $\|\M{A}\|_p$ denotes its entry wise $p$ norm which is defined  as $\|\M{A}\|_p = (\sum_{i=1}^n\sum_{j=1}^d |a_{ij}|^p)^{\frac{1}{p}}$. 
For $p=2$, this is also called the Frobenius norm and is also represented as $\|\M{A}\|_F$. 
We differentiate the matrix entry wise $p$ norm from the matrix induced operator $p$ norm which is denoted by the symbol $\|\M{A}\|_{(p)}$ and is defined as $\|\M{A}\|_{(p)}= \sup_{\M{x} \neq \M{0}}\frac{\|\M{Ax}\|_p}{\|\M{x}\|_p}$. 
Specifically for $p=2$, $\|\M{A}\|_{(2)} = \sigma_{\text{max}}(\M{A})$ i.e. the highest singular value of matrix $\M{A}$ and for $p=1$, $\|\M{A}\|_{(1)}$ is the highest absolute column sum of the matrix $\M{A}$. 
The Singular Value Decomposition $($abbreviated as SVD$)$ of a full rank matrix $\M{A}$ is given as $\M{A}=\M{U\Sigma V^T}$ where $\M{U} \in \mathbb{R}^{n \times d}$ and $\M{V} \in  \mathbb{R}^{d \times d}$ represent $\M{A}$'s left and right singular vectors respectively. 
$\M{\Sigma} \in \mathbb{R}^{d \times d}$ is a diagonal matrix containing the singular values of $\M{A}$ in descending order.
\subsection{Coresets}
A coreset is a small summary of data which can give provable guarantees for a particular optimization problem. 
Formally, given a dataset
$\M{X} \in \mathbb{R}^{n \times d}$ where the rows represent the datapoints, set of queries $Q$ and a non-negative cost function $\mathnormal{f}_{\M{q}}(\M{x})$ with parameter $\M{q} \in {Q}$ and data point $\M{x} \in \M{X}$, a dataset consisting of subsampled and appropriately reweighted points $\M{C}$ is an $\epsilon$-coreset if $\forall \M{q} \in Q$, 
\begin{equation*}
  \bigg|\sum_{\M{x} \in \M{X}}f_{\M{q}}(\M{x}) - \sum_{\tilde{\M{x}} \in \M{C}}{f}_{\M{q}}(\mathbf{\tilde{x}})\bigg| \leq \epsilon\sum_{\M{x} \in \M{X}}{f}_{\M{q}}(\M{x})  
\end{equation*}
 for some $\epsilon > 0$. We will denote $\sum_{\M{x} \in \M{X}}f_{\M{q}}(\M{x}) = f_\M{q}(\M{X})$. Even if the function is defined over entire dataset $\M{X}$ and not on individual points we can define coreset for it in terms of $f_\M{q}(\M{X})$ directly. For e.g. with $p\neq r$, $\|\M{Ax}\|_p^r$ cannot be represented in terms of sum on individual $\M{a_{i}}$'s but we can still define coreset property for it.
 When the above equation is satisfied  $\forall \M{q} \in Q$, $\M{C}$ is  called a strong coreset. We will refer only to a strong coreset as coreset in this paper. The advantage of
 a strong coreset, specifically in machine learning problems, is that we can train the model on the coreset, i.e. a smaller data set, and then use the model obtained from it as a surrogate for the model trained on the original data, with comparable accuracy. Formally, via \cite{bachem2017practical}, let $\M{C}$ be an $\epsilon$-coreset of $\M{X}$  for $\epsilon \in (0,\frac{1}{3})$ for some function. If $\M{q^{*}_X}$ and $\M{q^{*}_C}$ denote the optimal (infimum) solutions for $\M{X}$ and $\M{C}$ respectively, then $f_{\M{q^{*}_C}}(\M{X}) \leq (1+3\epsilon)f_{\M{q^{*}_X}}(\M{X})$. 
Given a dataset $\M{X}$, the query space $Q$, as well as the cost function $f_{\M{q}}(\cdot)$, Langberg et al.~\cite{langberg2010universal} define a set of scores,
termed as sensitivities, that can be used to create coresets via importance sampling. 
The sensitivity of the  $i^{th}$ point is defined as 
$s_{i} = \sup_{\M{q} \in Q} \frac{f_{\M{q}}(\M{x_i})}{\sum_{\M{x'} \in \M{X}} {{f}_\M{q}(\M{x'})}}$.
Intuitively sensitivity of a point captures its worst case contribution to the value of an objective function. The work by Langberg et al. \cite{langberg2010universal} shows that using any upper bounds to the sensitivity scores, 
we can create a probability distribution, using which we can then sample a coreset. Throughout the paper we refer to the sensitivity of the $i^{th}$ data point for some objective function as $s_i$ and the function will be clear from the context. Feldman et al. \cite{feldman2011unified} provided a unified framework to build coresets using the sensitivity framework which was improved by Braverman et al.~\cite{braverman2016new}. Their theorem on coreset size using sensitivity is the following:
\begin{theorem}\cite{braverman2016new}\label{th:loglinear core}
Let $f_{\M{q} \in Q}(\cdot)$ be the function, $\M{X}$ be the dataset and query space $Q$ be of dimension $d$. Let the sum of sensitivities be $S$. Let $(\epsilon,\delta) \in (0,1)$. Let $r$ be such that
\begin{equation*}
r \geq \frac{10S}{\epsilon^2}\bigg(d\log{S} + \log{\frac{1}{\delta}}\bigg).
\end{equation*}
$\M{C}$ be a matrix of $r$ rows, each 
sampled i.i.d from $\M{X}$ such that for every $\M{x_i} \in \M{X}$ and $\tilde{\M{x}}_i \in \M{C}$, $\tilde{\M{x}}_i$ is a scaled version of $\M{x_i}$ with probability $\frac{s_i}{S}$ and scaled by $\frac{S}{rs_i}$, then $\M{C}$ is an $\epsilon$-coreset of $\M{X}$ for function $f$, with probability at least $1- \delta$.
\end{theorem}
We present a slightly better version (in terms of dependence on $S$) of this theorem obtained by using a tighter tail inequality and use it for our coreset size. \begin{theorem}\label{th: linear core}
For the same setup as defined in Theorem \ref{th:loglinear core}, if $r = O{\big(\frac{S}{\epsilon^2}(d\log{\frac{1}{\epsilon}} + \log{\frac{1}{\delta}})\big)}$,  $\M{C}$ is an $\epsilon$-coreset of $\M{X}$ for function $f$, with probability at least $1- \delta$.
\end{theorem}
\begin{proof}
First we bound the sample size for a fixed query $\M{q} \in Q$. Let $s_i$ be the sensitivity of the $i^{th}$ point $\M{x_i}$ and $S$ be the sum of the sensitivities.
Let the sampling probability be $p_i =\frac{s_i}{S}$. 

For all $\M{q} \in Q$ and $\M{x_i} \in \M{X}$ define a function  
$g_{\M{q}}(\M{x_i}) =\frac{f_{\M{q}}(\M{x_i})}{Sp_i\sum_{j=1}^n{f_{\M{q}}(\M{x_j})}}$.                           
So, \[\mathbb{E}[g_{\M{q}}(\M{x_i})]=\frac{1}{S}\] for $\M{x_i}$ drawn uniformly at random from $\M{X}$ and 
\[ \frac{1}{r} \sum_{\substack{i \in [n] s.t. \\ \M{\tilde{x}_i}\in \M{C}}} g_{\M{q}}(\M{x_i}) =\frac{\sum_{\M{\tilde{x}_i} \in \M{C}}{f_{\M{q}}(\M{\tilde{x}_i})}} {S\sum_{\M{x_i} \in \M{X}}{f_{\M{q}}(\M{x_i})}} \]

Let, \[T = \sum_{\substack{i \in [n] s.t. \\ \M{\tilde{x}_i}\in \M{C}}} g_{\M{q}}(\M{x_i})\] then \[\mathbb{E}[T] = \sum_{\substack{i \in [n] s.t. \\ \M{\tilde{x}_i}\in \M{C}}} \mathbb{E}[g_{\M{q}}(\M{x_i})] = r/S\]

\begin{eqnarray*}
 \mbox{var}(g_{\M{q}}(\M{x_i})) &\leq& \mathbb{E}[(g_{\M{q}}(\M{x_i}))^{2}] \nonumber \\
 &=& \sum_{\M{x_i} \in \M{X}} \frac{(f_{\M{q}}(\M{x_i}))^{2}}{(\sum_{j=1}^{n}{f_{\M{q}}(\M{x_j}))^{2}}S^{2}p_i} \nonumber \\
 &\stackrel{i}{\leq} & \sum_{\M{x_i} \in \M{X}} \frac{(f_{\M{q}}(\M{x_i}))^{2}\sum_{j=1}^{n}{f_{\M{q}}(\M{x_j})}}{(\sum_{j=1}^{n}{f_{\M{q}}(\M{x_j}))^{2}}f_{\M{q}}(\M{x_i})S} \nonumber \\
 &=& 1/S
\end{eqnarray*}
We get $(i)$ by replacing values of $p_i$ and $s_i$.

Now $\mbox{var}(g_{\M{q}}(\M{x_i})) \leq \mathbb{E}[(g_{\M{q}}(\M{x_i}))^{2}] \leq 1/S$. So $\mbox{var}(T)\leq r/S$. 

Now applying Bernstein Inequality as given in \cite{dubhashi2009concentration} we get,
\begin{eqnarray*}
 \mbox{Pr}(| T-\mathbb{E}[T] | \geq r\epsilon') &\leq& \exp(-\frac{r^{2}\epsilon'^{2}}{r/S+r\epsilon'/3}) \nonumber \\
 \mbox{Pr}\bigg(\Big\vert\frac{\sum_{\M{\tilde{x}_i} \in \M{C}}{f_{\M{q}}(\M{\tilde{x}_i})}}{S\sum_{\M{x_i} \in \M{X}}{f_{\M{q}}(\M{x_i})}} - \frac{1}{S}\Big\vert \geq {\epsilon}'\bigg) &\leq& \exp\bigg(-\frac{r\epsilon'^{2}}{(1/S)+({\epsilon'}/3)}\bigg)
\end{eqnarray*}

Replacing $\epsilon'$ with $\epsilon/S $ we get,
$$\mbox{Pr}\bigg(\Big\vert\sum_{\M{\tilde{x}_i} \in \M{C}}{f_{\M{q}}(\M{\tilde{x}_i})}-\sum_{\M{x_i} \in \M{X}}{f_{\M{q}}(\M{x_i})} \Big\vert \geq \epsilon\sum_{\M{x_i} \in \M{X}}{f_{\M{q}}(\M{x_i})}\bigg) \leq 2\exp\bigg(\frac{-2r\epsilon^{2}}{S(1+\frac\epsilon 3)}\bigg)$$
To make the above probability less than $\delta$ we choose $r \geq \frac{S}{2\epsilon^2}(1+\frac\epsilon 3)\log\frac{2}{\delta}$ which depends on $S$ for a fixed query $\M{q} \in Q$.
Now to bound the number of samples required to give a uniform bound for all queries simultaneously $\forall {\M{q}} \in Q$, we use the same $\epsilon$-net argument as described in \cite{bachem2017practical}. This part is essentially a repeat of their argument. However we present it here for completeness. Observe that function $g_{\M{q}}(\M{x_i})$ lies in the interval $[0,1]$. Due to the bounded dimension $d$ of $Q$, the queries in $Q$ span a subspace $[0,1]^d$. There may be infinite number of queries in $Q$. However these may be covered up to $L_1$ distance $\epsilon/2$ by some set $Q^* \subset Q$ of $O(\epsilon^{-d})$ points \cite{haussler1995sphere} as given in \cite{bachem2017practical}. For the $\epsilon$-net argument let $\mathcal{E}$ be the bad event that the coreset property is not satisfied by some $\M{C}$. Therefore

\begin{eqnarray*}
 \mbox{Pr}(\mathcal{E})&=&\mbox{Pr}\bigg[\exists\M{q} \in Q: \Big\vert\sum_{\M{\tilde{x}_i} \in \M{C}}{f_{\M{q}}(\M{\tilde{x}_i})}-\sum_{\M{x_i} \in \M{X}}{f_{\M{q}}(\M{x_i})} \Big\vert > \epsilon\sum_{\M{x_i} \in \M{X}}{f_{\M{q}}(\M{x_i})}\bigg]\\
 &\leq&\mbox{Pr}\bigg[\exists\M{q} \in Q^*: \Big\vert\sum_{\M{\tilde{x}_i} \in \M{C}}{f_{\M{q}}(\M{\tilde{x}_i})}-\sum_{\M{x_i} \in \M{X}}{f_{\M{q}}(\M{x_i})} \Big\vert > \frac{\epsilon}{2}\sum_{\M{x_i} \in \M{X}}{f_{\M{q}}(\M{x_i})}\bigg]\\
 &\leq& 2|Q^*|\exp{\bigg(\frac{-2r\epsilon^{2}}{S(1+\frac\epsilon 3)}\bigg)}
\end{eqnarray*}
To make $\M{C}$ an $\epsilon$-coreset  with probabiltiy atleast 1-$\delta$, we choose
$r = O(\frac{S}{\epsilon^2}(\log{|Q^*|}+\log{\frac{2}{\delta}})$. Now as $|Q^*| \in O(\epsilon^{-d})$ we have $r = O{\big(\frac{S}{\epsilon^2}(d\log{\frac{1}{\epsilon}} + \log{\frac{1}{\delta}})\big)}$.  
\end{proof}
\subsection{$\ell_p$ Regression and Regularization}
The $\ell_p$ regression problem is defined as follows: given $\M{A}$ and $\M{b}$, find $\min _{\mathbf{x} \in \mathbb{R}^d} \|\mathbf{Ax} - \mathbf{b}\|_p^p$. 
For $p=1,2$ the problems are referred to as Least Absolute Deviation and Least Squares Regression  respectively.
We call a matrix $\M{\Pi}$ to have  {\em subspace preserving} property if, $\forall {\M{x}} \in \mathbb{R}^d$, $| \|\M{\Pi Ax}\|_p - \|\M{Ax}\|_p | \leq \epsilon \|\M{Ax}\|_p$.

Notice that if a sampling and reweighing matrix $\M{\Pi}$ satisfies the $\ell_p$ subspace preservation property, it will also provide a coreset for the $\ell_p$ regression problem.
If we create matrix $\M{A'}$ by concatenating $\M{A}$ and $\M{b}$ as $\M{A'}=\M{[A ~~ b]}$ and consider $\M{x'}=\begin{bmatrix} \mathbf{x}\\ -1 \end{bmatrix}$, then a subspace preserving property of $\M{\Pi}$ in $\mathbb{R}^{d+1}$ implies that  $\M{\Pi A}$ is a coreset for $\ell_p$ regression also.
For $p=2$, an  $O(\frac{d \log{d}}{\epsilon^2})$ sized coreset can be obtained by sampling  using the popular leverage scores which are the squared Euclidean row-norms of any orthogonal column basis of $\M{A}$~\cite{drineas2006sampling}.
A matrix $\M{U}$ is called an $(\alpha,\beta,p)$ well-conditioned basis for $\M{A}$ if $\|\M{U}\|_p \leq \alpha$ and $\forall \M{x} \in \mathbb{R}^d, \|\M{x}\|_q  \leq \beta \|\M{U}\M{x}\|_p $ where $\frac{1}{p}+\frac{1}{q}=1$. 
In fact, the $\M{U}$ obtained using the SVD of $\M{A}$ is a $(\sqrt{d},1,2)$ well-conditioned basis of $\M{A}$. 
For other values of $p>1$ \cite{dasgupta2009sampling} showed that by sampling using the $p^{th}$ power of the $p$ norm of rows of the $(\alpha,\beta,p)$ well-conditioned basis of $\M{A}$, we can obtain a coreset of size $\tilde{O}(\alpha\beta)^p$ with high probability for $\ell_p$ subspace preservation and hence $\ell_p$ regression. Well-conditioned basis for $p$ norm can be constructed in various ways like using Cauchy random variables~\cite{sohler2011subspace} or exponential random variables \cite{woodruff2013subspace}. Yang et al. provide a good review of methods \cite{yang2015implementing}. Our analysis works with any method of constructing the well-conditioned basis.
For the norm based regularized regression we consider the following general form for $\lambda > 0$
\begin{equation*}
    \min _{\M{x} \in \mathbb{R}^d} \|\M{A}\M{x} - \M{b}\|_p^r	+ \lambda \|\M{x}\|_q^s
\end{equation*}
for $p,q \geq 1$ and $r,s>0$. Notice that not all regularized regression forms falling under above category can be expressed as sum of individual functions. A  coreset for this problem is $(\tilde{\M{A}},\tilde{\M{b}})$ such that $\forall \M{x} \in \mathbb{R}^d$ and $\forall \lambda >0 $, 
	\begin{equation*}
	\|\M{\tilde{A}}{\M{x}} - \M{\tilde{b}}\|_p^r	+ \lambda \|{\M{x}}\|_q^s \in (1\pm\epsilon) (\|\M{A}\M{x} - \M{b}\|_p^r	+ \lambda \|\M{x}\|_q^s)
	\end{equation*} 
It is not difficult to verify that a coreset for the unregularized version of regression  is also a coreset for the regularized one. However, to reiterate, our goal is to analyze whether the size of coreset for the above form of regularized regression can be shown to be provably smaller than the size of the coreset for the unregularized counterpart. We answer this question in the following sections. We first show a negative result where the regularization does not result in a smaller coreset. Next, we show settings where the size of the coreset does decrease inversely with $\lambda$. 
Analogous to the subspace embedding, given a matrix $\M{A} \in \mathbb{R}^{n \times d}$,  a matrix $\M{A_c}$ with $k$ rows is a coreset for $(\|\mathbf{Ax}\|_p^r + \lambda \|\M{x}\|_q^s) $ if $\forall{\M{x}} \in \mathbb{R}^d$ and $\forall \lambda > 0$, the following holds.
\begin{equation}
\begin{split}
(1-\epsilon)(\|\mathbf{Ax}\|_p^r + \lambda \|\M{x}\|_q^s) \leq \|\mathbf{A_cx}\|_p^r + \lambda \|\M{x}\|_q^s \leq  (1+\epsilon)(\|\mathbf{Ax}\|_p^r + \lambda \|\M{x}\|_q^s)
\label{eq:regcore}
\end{split}
\end{equation}
It is easy to see that such a coreset will also give a coreset for the $\ell_p$ regression with $\ell_p$ regularization by using the concatenated matrix $\M{A'}$ and vector $\M{x'}$.
\section{Coresets For General Form of Regularized Regression }
For the case of ridge regression $(p,q,r,s=2)$, Avron et al. \cite{avron2017sharper} showed smaller coresets with size dependent on the statistical dimension of the matrix $\M{A}$. They constructed their coreset by sampling according to the ridge leverage scores. In this section we show that it is not always  possible to get a coreset for a regularized version of regression problem which is strictly smaller in size than the coreset for its unregularized counterpart. Note that, for the purposes of the following theorem,
a coreset of $\M{A}$ is any matrix $\M{A_c}$ that satisfies the coreset condition; in particular,  $\M{A_c}$ does not need to be a {\em sampling based coreset} of $\M{A}$. 

\begin{theorem}\label{th:ss preserve}Given a matrix $\M{A} \in \mathbb{R}^{n \times d}$ and  $\lambda >0$, any coreset for the problem $\|\mathbf{Ax}\|_p^r + \lambda \|\M{x}\|_q^s $, where $r \neq s$, $p,q \geq 1$ and $ r,s > 0 $, is also a coreset for  $\|\mathbf{Ax}\|_p^r $.
\end{theorem}
\begin{proof}
	The proof is by contradiction. Let $\M{A_c}$ be a coreset for $\|\mathbf{Ax}\|_p^r + \lambda \|\M{x}\|_q^s $, where $r \neq s$. Therefore, by definition of coreset, $\forall \M{x} \in \mathbb{R}^d$,
	$$ \|\mathbf{A_cx}\|_p^r + \lambda \|\M{x}\|_q^s \in (1 \pm \epsilon)(\|\mathbf{Ax}\|_p^r + \lambda \|\M{x}\|_q^s)$$
	Suppose $\M{A_c}$ is not a coreset for $\|\mathbf{Ax}\|_p^r $. We consider the two cases:\\
	\textbf{Case 1}:  $\exists  \M{x} \in \mathbb{R}^{d}$ s.t. $\|\mathbf{A_cx}\|_p^r > (1+\epsilon) \|\mathbf{Ax}\|_p^r$. Define $\epsilon'$ to be such that $\|\mathbf{A_cx}\|_p^r = (1+\epsilon')\|\mathbf{Ax}\|_p^r$. Clearly, $\epsilon' > \epsilon$. Let us define $\M{y}=\alpha\M{x}$ for some suitable $\alpha$. Consider the ratio
	\begin{eqnarray*}
	\frac{\|\mathbf{A_cy}\|_p^r + \lambda \|\M{y}\|_q^s}{\|\mathbf{Ay}\|_p^r + \lambda \|\M{y}\|_q^s}
	&=&\frac{\alpha^r \|\mathbf{A_c x}\|_p^r + \lambda\alpha^s \|\M{x}\|_q^s}{\alpha^r \|\mathbf{Ax}\|_p^r + \lambda\alpha^s \|\M{x}\|_q^s}\\
	&=&\frac{\alpha^r(1+\epsilon')\|\mathbf{Ax}\|_p^r + \lambda\alpha^s \|\M{x}\|_q^s}{\alpha^r \|\mathbf{Ax}\|_p^r + \lambda\alpha^s \|\M{x}\|_q^s}
	\end{eqnarray*}
	 Suppose $ r > s$. Then the ratio $ 
	      \frac{\alpha^r(1+\epsilon')\|\mathbf{Ax}\|_p^r + \lambda\alpha^s \|\M{x}\|_q^s}{\alpha^r \|\mathbf{Ax}\|_p^r + \lambda\alpha^s \|\M{x}\|_q^s}   = 
	    \frac{\alpha^{r-s}(1+\epsilon')\|\mathbf{Ax}\|_p^r + \lambda \|\M{x}\|_q^s}{\alpha^{r-s} \|\mathbf{Ax}\|_p^r + \lambda \|\M{x}\|_q^s} $. Here we want the ratio to be greater than $(1+ \epsilon)$.
	  By choosing $\alpha$ appropriately,  we can set the above ratio to be greater than $(1+(\frac{\epsilon + \epsilon'}{2})) > 1 + \epsilon$, since $\epsilon' > \epsilon$. A sufficient condition to choose $\alpha$ for this to happen is: $\alpha^{(r-s)} > \frac{(\epsilon' + \epsilon)}{(\epsilon'-\epsilon)} \frac{\lambda\|\M{x}\|_q^s}{\|\M{Ax}\|_p^r} $.
	  
	 Similarly if $r < s$, the ratio $
	     \frac{\alpha^r(1+\epsilon')\|\mathbf{Ax}\|_p^r + \lambda\alpha^s \|\M{x}\|_q^s}{\alpha^r \|\mathbf{Ax}\|_p^r + \lambda\alpha^s \|\M{x}\|_q^s}  =  \frac{(1+\epsilon')\|\mathbf{Ax}\|_p^r + \alpha^{s-r} \lambda \|\M{x}\|_q^s}{ \|\mathbf{Ax}\|_p^r + \alpha^{s-r}\lambda \|\M{x}\|_q^s}$. Here for the ratio to be greater than $(1+\epsilon)$, we can set $\alpha$ such that $\alpha^{(s-r)} < \frac{(\epsilon'-\epsilon)}{(\epsilon' + \epsilon)} \frac{\|\M{Ax}\|_p^r}{\lambda\|\M{x}\|_q^s}$.
	     
	 \textbf{Case 2}: $\exists  \M{x} \in \mathbb{R}^{d}$ s.t. $\|\mathbf{A_cx}\|_p^r < (1-\epsilon) \|\mathbf{Ax}\|_p^r$ and $\|\mathbf{A_cx}\|_p^r = (1-\epsilon')\|\mathbf{Ax}\|_p^r$ for some $\epsilon' > \epsilon$.
    Again define $\M{y}=\alpha\M{x}$ for some suitable $\alpha$. Consider the ratio
	\begin{eqnarray*}
	\frac{\|\mathbf{A_cy}\|_p^r + \lambda \|\M{y}\|_q^s}{\|\mathbf{Ay}\|_p^r + \lambda \|\M{y}\|_q^s} &=& \frac{\alpha^r \|\mathbf{A_cx}\|_p^r + \lambda\alpha^s \|\M{x}\|_q^s}{\alpha^r \|\mathbf{Ax}\|_p^r + \lambda\alpha^s \|\M{x}\|_q^s}\\
	&=&\frac{\alpha^r(1-\epsilon')\|\mathbf{Ax}\|_p^r + \lambda\alpha^s \|\M{x}\|_q^s}{\alpha^r \|\mathbf{Ax}\|_p^r + \lambda\alpha^s \|\M{x}\|_q^s}
	\end{eqnarray*}
	Suppose $r > s$. Here the ratio $\frac{\alpha^r(1-\epsilon')\|\mathbf{Ax}\|_p^r + \lambda\alpha^s \|\M{x}\|_q^s}{\alpha^r \|\mathbf{Ax}\|_p^r + \lambda\alpha^s \|\M{x}\|_q^s}  = \frac{\alpha^{r-s}(1-\epsilon')\|\mathbf{Ax}\|_p^r + \lambda \|\M{x}\|_q^s}{\alpha^{r-s} \|\mathbf{Ax}\|_p^r + \lambda \|\M{x}\|_q^s}$. We want the ratio to be smaller than $(1-\epsilon)$. Without loss of generality, we can set the ratio to be smaller than $(1-(\frac{\epsilon + \epsilon'}{2}))$ since $\epsilon' > \epsilon$. For this can set $\alpha$ s.t. $\alpha^{(r-s)} > \frac{(\epsilon' + \epsilon)}{(\epsilon'-\epsilon)}\frac{\lambda\|\M{x}\|_q^s}{\|\M{Ax}\|_p^r} $.\\
	 Similarly if $ r < s $, the ratio $\frac{\alpha^r(1-\epsilon')\|\mathbf{Ax}\|_p^r + \lambda\alpha^s \|\M{x}\|_q^s}{\alpha^r \|\mathbf{Ax}\|_p^r + \lambda\alpha^s \|\M{x}\|_q^s}  = \frac{(1-\epsilon')\|\mathbf{Ax}\|_p^r + \alpha^{s-r} \lambda \|\M{x}\|_q^s}{ \|\mathbf{Ax}\|_p^r + \alpha^{s-r}\lambda \|\M{x}\|_q^s}$. Here for the ratio to be smaller than $(1-\epsilon)$, we can set $\alpha$ such that $\alpha^{(s-r)} < \frac{(\epsilon' - \epsilon)}{(\epsilon'+\epsilon)}\frac{\|\M{Ax}\|_p^r}{\lambda\|\M{x}\|_q^s}$. Hence in both the cases, for both scenarios of $r$ and $s$ we can set an $\alpha$ which gives us a contradiction to the fact that $\M{A_c}$ is a coreset for the regularized function. Hence our assumption is wrong and $\M{A_c}$ is also a coreset for the unregularized function.
\end{proof}
This theorem implies the following corollary which gives our impossibility result.
\begin{corollary}\label{cor: impos res}
	Given a matrix $\M{A} \in \mathbb{R}^{n \times d} $ and a corresponding vector $\M{b} \in \mathbb{R}^{n}$, then for the function $\|\M{A}\M{x} - \M{b}\|_p^r	+ \lambda \|\M{x}\|_q^s$, $ r \neq s$, we cannot get a coreset of size smaller than the size of the optimal sized coreset for $\|\M{A}\M{x} - \M{b}\|_p^r.$
	\begin{proof}
	Consider the case $\M{b} = 0$, which reduces the above function to  $\|\M{A}\M{x}\|_p^r	+ \lambda \|\M{x}\|_q^s$. Theorem \ref{th:ss preserve} then implies that no smaller coreset is possible for the regularized problem. 
	This proof can be generalized to the setting when $\M{b}$ is in the column-space of $\M{A}$ in the following manner.
	Suppose $\M{b} = \M{A}\M{u}$. Also suppose $\M{A_c}$ and $\M{b_c}$ can be obtained as $\M{A_c} = \M{SA}$ and $\M{b_c}=\M{Sb}$ where $\M{S}$ can be either a sampling and reweighing or a scaling matrix.
	Now we want to prove the following : If $\M{S}$ is a coreset creation matrix for $(\M{A},\M{b})$ for regression i.e. $\forall{\M{x}}, \|\M{A_cx}-\M{b_c}\|_p^r \in (1 \pm \epsilon) \|\M{Ax}-\M{b}\|_p^r$, then it must be that
	$\forall{\M{x}}, \|\M{A_cx}\|_p^r \in (1 \pm \epsilon) \|\M{Ax}\|_p^r$.
	Proving this statement and using Theorem \ref{th:ss preserve} essentially proves the corollary for the more general setting of $\M{b}$ in column space of $\M{A}$. 
	To prove the statement we use contradiction. Let us suppose that the statement is false. Then $\exists \M{v} \in \mathbb{R}^d$ s.t. $\|\M{A_cv}\|_p^r > (1 + \epsilon) \|\M{Av}\|_p^r$. We will create a $\M{y}$ s.t that $\|\M{A_cy}-\M{b_c}\|_p^r > (1+\epsilon) \|\M{Ay}-\M{b}\|_p^r$. Consider the ratio
	\begin{equation*}
	    \frac{\|\M{S}(\M{Ay}-\M{b})\|_p^r}{\|\M{Ay}-\M{b}\|_p^r} = \frac{\|\M{S}\M{A}(\M{y}-\M{u})\|_p^r}{\|\M{A}(\M{y}-\M{u})\|_p^r}
	\end{equation*}
	Now if we choose $\M{y}= \M{u} + \M{v}$ then we have $\frac{\|\M{SAv}\|_p^r}{\|\M{Av}\|_p^r} > (1 + \epsilon) $. This a contradiction to the fact that $\M{SA},\M{Sb}$ is coreset to $\|\M{Ay}-\M{b}\|_p^r$. Hence our assumption is false. So $\forall{\M{x}}, \|\M{A_cx}\|_p^r \leq (1+\epsilon) \|\M{Ax}\|_p^r$. The other direction for coreset definition is proved in similar manner. This combined with Theorem \ref{th:ss preserve} gives our corollary
   \end{proof}
\end{corollary}
\section{The Modified Lasso}
Given  a matrix $ \mathbf{A} \in {\mathbb{R}}^{n \M{x}d}$ with rank $d$ and $n \gg d$ and a vector $\mathbf{b} \in \mathbb{R} ^ n $ and $\lambda > 0$, the lasso problem is stated as
\begin{eqnarray*}
&\min _{\mathbf{x} \in \mathbb{R}^d}& \|\mathbf{Ax} - \mathbf{b}\|_2^2 + \lambda \|\M{x}\|_1\\
=&\min _{\mathbf{x} \in \mathbb{R}^d}& {\sum_{i=1}^n (\mathbf{a}_i^T\mathbf{x} - b_i)^2} + \lambda~\|\M{x}\|_1
\end{eqnarray*}
However as $r=2$ and $s=1$ , by corollary \ref{cor: impos res} we can not hope to get a coreset smaller than the one for least square regression. 

To  preserve the sparsity inducing nature of the lasso problem and still obtain a smaller sized coreset, we consider a slightly different version of lasso  which we call {\em modified lasso}. It is stated as follows:
\begin{eqnarray*}
&\min _{\mathbf{x} \in \mathbb{R}^d}& \|\mathbf{Ax} - \mathbf{b}\|_2^2 + \lambda \|\M{x}\|_1^2 \\
=&\min _{\mathbf{x} \in \mathbb{R}^d}& {\sum_{i=1}^n (\mathbf{a}_i^T\mathbf{x} - b_i)^2} + \lambda~\|\M{x}\|_1^2
\end{eqnarray*}

Note that in the constraint based formulation (i.e. least square with constraint $\|\M{x}\|_1 \le R$), the normal lasso and the modified one are the same with a change in the constraint radius $R$. In our experiments, we will also empirically show that just like lasso, the modified version also induces sparsity in the solution vector. However, as we will see, the behaviour of the regularized versions of the two problems are different with respect to the optimal coreset size.  

We now show a stronger result which leads us to a smaller coreset for the modified lasso problem.
\begin{theorem}
Given a matrix $\M{A} \in \mathbb{R}^{n \times d}$, corresponding vector $\M{b} \in \mathbb{R}^n$, any coreset for the function $\|\M{Ax-b}\|_p^p + \lambda\|\M{x}\|_p^p$ is also a coreset of the function $\|\M{Ax-b}\|_p^p + \lambda\|\M{x}\|_q^p$ where $q\leq p$, $p,q \geq 1.$
\label{thm:pqcoreset}
\end{theorem}
\begin{proof}
    Let $(\M{A_c}, \M{b_c})$ be a coreset for the function $\|\M{Ax-b}\|_p^p + \lambda\|\M{x}\|_p^p$. Hence,
    \begin{align*}
    &\|\M{A_cx-b_c}\|_p^p + \lambda\|\M{x}\|_q^p\\
    &=\|\M{A_cx-b_c}\|_p^p + \lambda\|\M{x}\|_p^p - \lambda\|\M{x}\|_p^p +  \lambda\|\M{x}\|_q^p\\
    &\leq (1+ \epsilon)(\|\M{Ax-b}\|_p^p + \lambda\|\M{x}\|_p^p)-\lambda\|\M{x}\|_p^p +  \lambda\|\M{x}\|_q^p\\ 
    &=(1+\epsilon)\|\M{Ax-b}\|_p^p + \epsilon\lambda\|\M{x}\|_p^p+\lambda\|\M{x}\|_q^p\\
    &\leq (1+\epsilon)\|\M{Ax-b}\|_p^p + \epsilon\lambda\|\M{x}\|_q^p+\lambda\|\M{x}\|_q^p\\
    &= (1+ \epsilon)(\|\M{Ax-b}\|_p^p + \lambda\|\M{x}\|_q^p)
    \end{align*}
    The second inequality follows from the fact that for any $\M{x} \in \mathbb{R}^d$, $\|\M{x}\|_q \geq \|\M{x}\|_p$ for $p\geq q$. This proves one direction in the definition of coreset. For the other direction consider
    \begin{align*}
     &\|\M{A_cx-b_c}\|_p^p + \lambda\|\M{x}\|_q^p\\
     &=\|\M{A_cx-b_c}\|_p^p + \lambda\|\M{x}\|_p^p - \lambda\|\M{x}\|_p^p +  \lambda\|\M{x}\|_q^p\\
     &\geq (1- \epsilon)(\|\M{Ax-b}\|_p^p + \lambda\|\M{x}\|_p^p) - \lambda\|\M{x}\|_p^p +  \lambda\|\M{x}\|_q^p\\
     &=(1-\epsilon) \|\M{Ax-b}\|_p^p - \epsilon\lambda\|\M{x}\|_p^p +\lambda\|\M{x}\|_q^p\\
     &\geq (1-\epsilon) \|\M{Ax-b}\|_p^p - \epsilon\lambda\|\M{x}\|_q^p +\lambda\|\M{x}\|_q^p\\
     &= (1-\epsilon)(\|\M{Ax-b}\|_p^p + \lambda\|\M{x}\|_q^p)
    \end{align*}
    Combining both inequalities we get the result.
\end{proof}

As a result of the above Theorem~\ref{thm:pqcoreset}, we get the following corollary about the modified lasso problem.

\begin{corollary}
	For the modified lasso problem there exists an $\epsilon$-coreset of size $O(\frac{sd_\lambda(\M{A}) \log{sd_\lambda(\M{A})}}{\epsilon^2})$ with a high probability where $sd_\lambda(\M{A})=\sum_{j\in[d]}{\frac{1}{1+\frac{\lambda}{\sigma_j^2}}}$ is the statistical dimension of matrix $\M{A}$ which has singular values $\sigma_j$'s.
\end{corollary}
\begin{proof}
    For $p=2$ and $q=1$ we get the statement of Theorem \ref{thm:pqcoreset} for the ridge regression and the modified lasso problem. By application of the theorem, a coreset for ridge regression is also a coreset for the modified lasso problem for above values of $p$ and $q$. Using the results of ~\cite{avron2017sharper} we can show that by sampling points according to the ridge leverage scores we can get a coreset of size $O(\frac{sd_\lambda(\M{A}) \log{sd_\lambda(\M{A})}}{\epsilon^2})$  for ridge regression with constant probability and hence the corollary.
\end{proof}
It is important to note that the motivation behind the modified lasso proposal is purely computational. It can have smaller coresets and as the constrained version of both problems are similar, we have proposed it for settings where it is desirable to solve the optimization problem on a coreset for scalability reasons. To illustrate, solving either the modified lasso or lasso on a $1L \times 30$ matrix takes roughly 80-90 seconds using the same algorithm.
For the modified version a $200 \times 30$ coreset can be constructed and problem can be solved in about 1.6 secs with a relative error of $0.0098$.
Understanding the statistical properties of modified lasso is an interesting open direction.
\section{The $\ell_p$ Regression with $\ell_p$ Regularization}
The $\ell_p$ Regression with $\ell_p$ Regularization is given as 
\begin{eqnarray*}
&\min _{\mathbf{x} \in \mathbb{R}^d}& \|\mathbf{Ax} - \mathbf{b}\|_p^p + \lambda \|\M{x}\|_p^p\\
=&\	\min _{\mathbf{x} \in \mathbb{R}^d}& {\sum_{i=1}^n |(\mathbf{a}_i^T\mathbf{x} - b_i)|^p} + \lambda~\|\M{x}\|_p^p
\end{eqnarray*}
For this problem our main result is as follows
\begin{theorem}\label{th: lprlpr}
For $\ell_p$ regression with $\ell_p$ regularization, there is a coreset of size ${O}\bigg(\frac{(\alpha\beta)^pd\log{\frac{1}{\epsilon}}}{\big(1+\frac{\lambda}{\|\M{A'}\|_{(p)}^p}\big)\epsilon^2}\bigg)$ with high probability. Here $\M{A'} = [\M{A} ~~ \M{b}]$ and has an $(\alpha,\beta,p)$ well-conditioned basis.
\end{theorem}
\begin{proof}
We define the sensitivity of the $i^{th}$ point for the $\ell_p$ regression with $\ell_p$ regularization  problem as follows:
$
s_i = \sup_\M{x'}  \frac{|\M{a'}_i^T\M{x'}|^p + \frac{\lambda\|\M{x'}\|_p^p}{n}}{ \sum_{j}|\M{a'}_j^T\M{x'}|^p +\lambda\|\M{x'}\|_p^p}$.
Again, $\M{A'}$ here is the concatenated matrix and $\M{x'}$ is the concatenated vector. Let $\M{A'}=\M{UV}$, then $\M{a'}_j^T=\M{u}_j^T\M{V}$. Here $\M{U}$ is an $(\alpha,\beta,p)$ well-conditioned basis for $\M{A'}$. Now let $\M{a'}_j^T\M{x'}=\M{u}_j^T\M{Vx'}= \M{{u}_j^T\M{z}}$. So
\begin{equation*}
\begin{split}
s_{i} &= \sup_\M{z}  \frac{|\M{u}_i^T\M{z}|^p + \frac{\lambda}{n}\|\M{V}^{-1}\M{z}\|_p^p}{\sum_{j}|\M{u}_j^T\M{z}|^p +\lambda\|\M{V}^{-1}\M{z}\|_p^p}\\
&= \sup_\M{z}\Bigg( \frac{|\M{u}_i^T\M{z}|^p}{\sum_{j}|\M{u}_j^T\M{z}|^p +\lambda\|\M{V}^{-1}\M{z}\|_p^p} \\ 
&+ \frac{\frac{\lambda}{n}\|\M{V}^{-1}\M{z}\|_p^p}{\sum_{j}|\M{u}_j^T\M{z}|^p +\lambda\|\M{V}^{-1}\M{z}\|_p^p}\Bigg)\\
&\leq \sup_\M{z}  \frac{|\M{u}_i^T\M{z}|^p}{\sum_{j}|\M{u}_j^T\M{z}|^p +\lambda\|\M{V}^{-1}\M{z}\|_p^p} + \frac{1}{n}
\end{split}
\end{equation*}
$\M{U}$ is an $(\alpha,\beta,p)$ well-conditioned basis for $\M{A'}$. Hence by definition $\|\M{U}\|_p \leq \alpha$ and $\forall \M{z} \in \mathbb{R}^{d+1}, \|\M{z}\|_q  \leq \beta \|\M{U}\M{z}\|_p $ where $\frac{1}{p}+\frac{1}{q}=1$. Now the first term $\frac{|\M{u}_i^T\M{z}|^p}{\sum_{j}|\M{u}_j^T\M{z}|^p +\lambda\|\M{V}^{-1}\M{z}\|_p^p}$ lies in the interval $\Big[\frac{|\M{u}_i^T\M{z}|^p}{\alpha^p\|\M{z}\|_q^p +\lambda\|\M{V}^{-1}\M{z}\|_p^p},\frac{|\M{u}_i^T\M{z}|^p}{\frac{1}{\beta^p}\|\M{z}\|_q^p +\lambda\|\M{V}^{-1}\M{z}\|_p^p}\Big]$. This is by the application of Holder's inequality and the definition of $\beta$. Now instead of calculating supremum over  $\frac{|\M{u}_i^T\M{z}|^p}{\frac{1}{\beta^p}\|\M{z}\|_q^p +\lambda\|\M{V}^{-1}\M{z}\|_p^p}$, we calculate the infimum over its reciprocal. Lets call it $m$
\begin{equation*}
\begin{split}
m&=\inf_{\M{z}} \bigg(\frac{\|\M{z}\|_q^p}{\beta^p|\M{u}_j^T\M{z}|^p} + \frac{\lambda\|\M{V^{-1}z}\|_p^p}{|\M{u}_j^T\M{z}|^p}\bigg)\\
&\geq \frac{1}{\beta^p} \inf_{\M{z}}\frac{\|\M{z}\|_q^p}{\|\M{u}_j\|_p^p\|\M{z}\|_q^p} + 
\lambda \inf_{\M{z}}\frac{\|\M{V^{-1}z}\|_p^p}{|\M{u}_j^T\M{z}|^p}\\
&\geq \frac{1}{\beta^p\|\M{u}_j\|_p^p} + \lambda \inf_{\M{z}}\frac{\|\M{V^{-1}z}\|_p^p}{|\M{u}_j^T\M{z}|^p}
\end{split}
\end{equation*}
Now sensitivity of $i^{th}$ point is bounded as $s_i \leq \frac{1}{m} +\frac{1}{n}$. Therefore $s_i \leq \frac{1}{\frac{1}{\beta^p\|\M{u}_j\|_p^p} + \lambda  \inf_{\M{z}}\frac{\|\M{V^{-1}z}\|_p^p}{|\M{u}_j^T\M{z}|^p}} + \frac{1}{n}$. 

Now let us consider
$$m'=\inf_{\M{z}}\frac{\|\M{A'V^{-1}z}\|_p^p}{|\M{u}_j^T\M{z}|^p}
\geq \inf_{\M{z}}\frac{\|\M{A'V^{-1}z}\|_p^p}{\|\M{u}_j\|_p^p\|\M{z}\|_q^p}
\geq \frac{1}{\beta^p \|\M{u}_j\|_p^p}$$
Also  $\|\M{A'V^{-1}z}\|_p \leq \|\M{A'}\|_{(p)}\|\M{V^{-1}z}\|_p$. Hence $\|\M{V^{-1}z}\|_p \geq \frac{\|\M{A'V^{-1}z}\|_p}{\|\M{A'}\|_{(p)}}$. Here $\|\M{A'}\|_{(p)}$ represents the induced matrix norm defined as $\|\M{A'}\|_{(p)}= \sup_{\M{x'} \neq \M{0}}\frac{\|\M{A'x'}\|_p}{\|\M{x'}\|_p}$. Combining this with the inequality for $m$ we get
\begin{equation*}
m \geq \frac{1}{\beta^p\|\M{u}_j\|_p^p}\bigg(1+ \frac{\lambda}{\|\M{A'}\|_{(p)}^p}\bigg) 
\end{equation*}
This results in $s_i \leq \frac{\beta^p\|\M{u}_i\|_p^p}{1+ \frac{\lambda}{\|\M{A'}\|_{(p)}^p}}+ \frac{1}{n}$.  So the sum of sensitivities is bounded by $S \leq \frac{(\alpha\beta)^p}{1+ \frac{\lambda}{\|\M{A'}\|_{(p)}^p}} +1 $.
Using the probability distribution based on sensitivity upper bounds defined in the proof to sample rows of $\M{A'}$ and applying Theorem~\ref{th: linear core}, we get the result. 
\end{proof}
As the value is decreasing in $\lambda$, this coreset is smaller than the coreset obtained for $\ell_p$ regression using just the well-conditioned basis. This method works with any well-conditioned basis and sampling complexity is dependent on the quality of the well-conditioned basis. For e.g.: for $p >2$ we can get $\alpha\beta = d^{1/p + 1/q}$. So we can get a coreset of size $O\bigg(\frac{d^{p+1}\log{\frac{1}{\epsilon}}}{\big(1+\frac{\lambda}{\|\M{A'}\|_{(p)}^p}\big)\epsilon^2}\bigg)$ with high probability. For the specific case of $p = 2$, the well-conditioned basis is an orthogonal basis (e.g. singular vectors) of $\M{A}$ and it is easy to modify only a few steps of the given proof to obtain the same result as \cite{avron2017sharper}. However in \cite{avron2017sharper} final sampling size is obtained using randomized matrix multiplication analysis which applies only for $p=2$. This is true for results on unregularized version of regression also available in literature.\\
Let us consider the special case of the Regularized Least Absolute Deviation problem.
\subsection{Regularized Least Absolute Deviation (RLAD)}
The RLAD problem given as 
\begin{equation*}
\min _{\mathbf{x} \in \mathbb{R}^d} \|\mathbf{Ax} - \mathbf{b}\|_1 + \lambda \|\M{x}\|_1\\
\end{equation*} 
has the benefits of robustness of $\ell_1$ regression and sparsity inducing nature of lasso \cite{wang2006regularized}. Plugging in $p=1$ in the above analysis we get the following upper bound for sensitivity of the RLAD problem: 
$s_i \leq \frac{\beta\|\M{u}_i\|_1}{1+ \frac{\lambda}{\|\M{A'}\|_{(1)}}}+\frac{1}{n} = \frac{\beta\|\M{u}_i\|_1}{1+ \frac{\lambda}{\max_{k \in [d]}\M{\|a'^{k}}\|_1}} + \frac{1}{n}$ where $\M{a'^k}$ is the $k^{th}$ column of $\M{A'}$. So the sum of sensitivities is bounded by $S \leq \frac{\alpha\beta}{1+ \frac{\lambda}{\max_{k \in [d]}\M{\|a'}^{k}\|_1}} +1 $. This implies a coreset size of  ${O}\bigg(\frac{(\alpha\beta)d\log{\frac{1}{\epsilon}}}{\big(1+\frac{\lambda}{\|\M{A'}\|_{(1)}}\big)\epsilon^2}\bigg)$ for the RLAD problem with high probability by Theorem~\ref{th: linear core}. This is smaller than the size of coreset ${O}(\frac{(\alpha\beta)d}{\epsilon^2})$ obtained for $\ell_1$ regression using well-conditioned basis. For $p=1$, we can get a well-conditioned basis with $\alpha\beta=d^{3/2}$. So we can get a coreset of size $O\bigg(\frac{d^{5/2}\log{\frac{1}{\epsilon}}}{\big(1+\frac{\lambda}{\|\M{A'}\|_{(1)}}\big)\epsilon^2}\bigg)$.
\subsection{Coresets for Multiresponse Regularized Regression}
Consider the multiresponse RLAD problem given as
\begin{equation*}
    \min_{\M{X} \in \mathbb{R}^{d \times k}}{\|\M{AX-B}\|_1 + \lambda\|\M{X}\|_1}
\end{equation*}
Here $\M{B} \in \mathbb{R}^{n \times k}$ i.e. there are $k$ different responses and hence our solution is also a matrix $\M{X}$. We regularize the problem with entry wise $1$ norm of the solution matrix. Let us consider the concatenated matrices $\hat{\M{A}}= [\M{A} ~~ \M{-B}]$ and $\hat{\M{X}}=\begin{bmatrix} \mathbf{X}\\ \M{I_k} \end{bmatrix}$ where $\M{I_k}$ is $k$-dimensional identity matrix. Using this matrices we define the sensitivity of the multiresponse RLAD problem as $s_i=\sup_{\hat{\M{X}}}  \frac{\|\hat{\M{a}}_i^T\hat{\M{X}}\|_1 + \frac{\lambda\|\hat{\M{X}}\|_1}{n}}{ \sum_{j}\|\hat{\M{a}}_j^T\hat{\M{X}}\|_1 +\lambda\|\hat{\M{X}}\|_1}$. 
We can get the sensitivity upper bounds for multiple response RLAD as we obtained for the single response RLAD. This gives us the following result
\begin{corollary}\label{cor:mreslpr}
For multiple response RLAD there exists a coreset of size ${O}\bigg(\frac{(\alpha\beta)dk\log{\frac{1}{\epsilon}}}{\big(1+\frac{\lambda}{\|\hat{\M{A}}\|_{(1)}}\big)\epsilon^2}\bigg)$ with high probability where we have an $(\alpha, \beta, 1)$ well-conditioned basis of $\M{\hat{A}}$.
\end{corollary}
The proof is similar to the one for the single response case. This can be extended to other values of $p$.
\begin{proof}
For $\hat{\M{A}}= [\M{A} ~~ \M{-B}]$ and $\hat{\M{X}}=\begin{bmatrix} \mathbf{X}\\ \M{I_k} \end{bmatrix}$ where $\M{I_k}$ is $k$-dimensional identity matrix, the sensitivity of Multiresponse RLAD problem is given as 
\begin{equation*}s_i=\sup_{\hat{\M{X}}}  \frac{\|\hat{\M{a}}_i^T\hat{\M{X}}\|_1 + \frac{\lambda\|\hat{\M{X}}\|_1}{n}}{ \sum_{j}\|\hat{\M{a}}_j^T\hat{\M{X}}\|_1 +\lambda\|\hat{\M{X}}\|_1}
\end{equation*}
Let $\hat{\M{A}}=\M{UY}$ where $\M{U}$ is an $(\alpha,\beta,1)$ well conditioned basis for $\hat{\M{A}}$. So $\hat{\M{a}}_j^T\hat{\M{X}}=\M{u_j}^T\M{Y}\hat{\M{X}}$. Let $\M{Y}\hat{\M{X}}=\M{Z}$. So the sensitivity equation becomes
\begin{eqnarray*}
   s_i&=\sup_{\M{Z}}  \frac{\|{\M{u}}_i^T{\M{Z}}\|_1 + \frac{\lambda\|{\M{Y}^{-1}\M{Z}}\|_1}{n}}{ \sum_{j}\|{\M{u}}_j^T{\M{Z}}\|_1 +\lambda\|{\M{Y}^{-1}\M{Z}}\|_1}\\
   &\leq \sup_{\M{Z}} \frac{\|{\M{u}}_i^T{\M{Z}}\|_1}{\sum_{j}\|{\M{u}}_j^T{\M{Z}}\|_1 +\lambda\|{\M{Y}^{-1}\M{Z}}\|_1} + \frac{1}{n}
\end{eqnarray*} 
Instead of supremum of the first quantity on the right hand size, we take the infimum of its reciprocal. Lets call it $m$.
\begin{eqnarray*}
m&=\inf_{\M{Z}} \frac{\sum_{j}\|{\M{u}}_j^T{\M{Z}}\|_1 +\lambda\|{\M{Y}^{-1}\M{Z}}\|_1}{\|{\M{u}}_i^T{\M{Z}}\|_1} \\
& \geq \inf_{\M{Z}}\frac{\sum_{j}\|{\M{u}}_j^T{\M{Z}}\|_1}{\|{\M{u}}_i^T{\M{Z}}\|_1} + \inf_{\M{Z}}\frac{\lambda\|{\M{Y}^{-1}\M{Z}}\|_1}{\|{\M{u}}_i^T{\M{Z}}\|_1}
\end{eqnarray*}
Let us consider the first part. $\M{U}$ is an $(\alpha,\beta,1)$- well conditioned basis for $\hat{\M{A}}$. Hence by definition $\|\M{U}\|_1 \leq \alpha$ and $\forall \M{z} \in \mathbb{R}^{d+k}, \|\M{z}\|_\infty  \leq \beta \|\M{U}\M{z}\|_1 $ . So the first term in the infimum
\begin{align*}
    &\inf_{\M{Z}}\frac{\sum_{j}\|{\M{u}}_j^T{\M{Z}}\|_1}{\|{\M{u}}_i^T{\M{Z}}\|_1}\\
    &= \inf_{\M{Z}}\frac{\sum_{l=1}^{k}{\|\M{U}\M{z}^l\|_1}}{\sum_{l=1}^{k}|\M{u_i}^T\M{z}^l|}\\
    &\geq \frac{\frac{1}{\beta}\sum_{l=1}^{k}\|\M{z}^l\|_\infty}{\|\M{u_i}\|_1\sum_{l=1}^{k}\|\M{z}^l\|_\infty}\\
    &=\frac{1}{\beta\|\M{u_i}\|_1}
\end{align*}
Now for the second term in the infimum let us consider instead
\begin{align*}
    &\inf_{\M{Z}}\frac{\|\M{AY^{-1}Z}\|_1}{\|{\M{u}}_i^T{\M{Z}}\|_1}\\
    &=\inf_{\M{Z}}\frac{\|\M{UZ}\|_1}{\|{\M{u}}_i^T{\M{Z}}\|_1}\\
    &\geq\frac{1}{\beta\|\M{u_i}\|_1}
\end{align*}
Now $\|\M{AY^{-1}Z}\|_1 \leq \|\M{A}\|_{(1)}\|\M{Y^{-1}Z}\|_1$. Therefore
\begin{align*}
    &\inf_{\M{Z}}\frac{\|\M{Y^{-1}Z}\|_1}{\|{\M{u}}_i^T{\M{Z}}\|_1}\\
    &\geq \inf_{\M{Z}}\frac{\|\M{AY^{-1}Z}\|_1}{{\|\M{A}\|_{(1)}\|\M{u}}_i^T{\M{Z}}\|_1}\\
    &\geq\frac{1}{\beta\|\M{A}\|_{(1)}\|\M{u_i}\|_1}
\end{align*}
Combining both these 
\begin{equation*}
    m \geq \frac{1}{\beta\|\M{u_i}\|_1}\bigg(1+\frac{\lambda}{\|\M{A}\|_{(1)}}\bigg)
\end{equation*}
Now sensitivity of $i^{th}$ point is bounded as $s_i \leq \frac{1}{m} +\frac{1}{n}$. Therefore $s_i \leq \frac{\beta\|\M{u}_i\|_1}{1+ \frac{\lambda}{\|\M{A}\|_{(1)}}}+\frac{1}{n}$. So the sum of sensitivities is bounded by $S \leq \frac{\alpha\beta}{1+ \frac{\lambda}{\|\M{A}\|_{(1)}}} +1 $. This fact combined with fact that dimension of $\M{X}$ is $dk$ and applying Theorem \ref{th: linear core} proves the corollary
\end{proof}
\section{Experiments}
In this section we describe the empirical results supporting our claims. We performed experiments for the modified lasso and the RLAD problem. We generated a matrix $\M{A}$ of size $100000 \times 30$ in which there a few rows with high leverage scores. The construction of this matrix is described in \cite{yang2015implementing} where they refer to it as an NG matrix. The NG (non uniform leverage scores with good condition number) matrix is generated by the following matlab command: \begin{verbatim}
    NG=[alpha*randn(n-d/2,d/2) (10^-8)*rand(n-d/2,d/2);zeros(d/2,d/2) eye(d/2)];
\end{verbatim}
Here we used $alpha=0.00065$ to get a condition number of about $5$. A solution vector $\M{x} \in \mathbb{R}^{30}$ was generated randomly and a response vector $\M{b= Ax} + (10^{-5})\frac{\|\M{b}\|_2}{\|\M{e}\|_2}\M{e}$ where $\M{e}$ is a vector of noise. All the experiments were performed in \textbf{MatlabR2017a} on a machine with 16GB memory and 8 cores of 3.40 GHz. In our first experiment we solved the modified lasso problem on the entire data matrix $\M{A}$ and response vector $\M{b}$ to see the effect on sparsity of solution for different values of $\lambda$. We also solved the lasso and the ridge problem on the same data and compared the number of zeros in the solution vector for different values of $\lambda$. To take numerical precision into account we converted each coordinate of the vector having absolute value less than $10^{-6}$ to $0$ and then plotted the number of zeros against $\lambda$ for each problem. As seen in Figure \ref{ModifiedLASSOsparse}, just like lasso, the modified lasso also induces sparsity in the solution and hence can be used as an alternate to lasso. As expected, no sparsity was induced by ridge regression so it is not seen on the graph.
\begin{figure}[ht]
\vskip 0.2in
\begin{center}
\centerline{\includegraphics[width=\columnwidth]{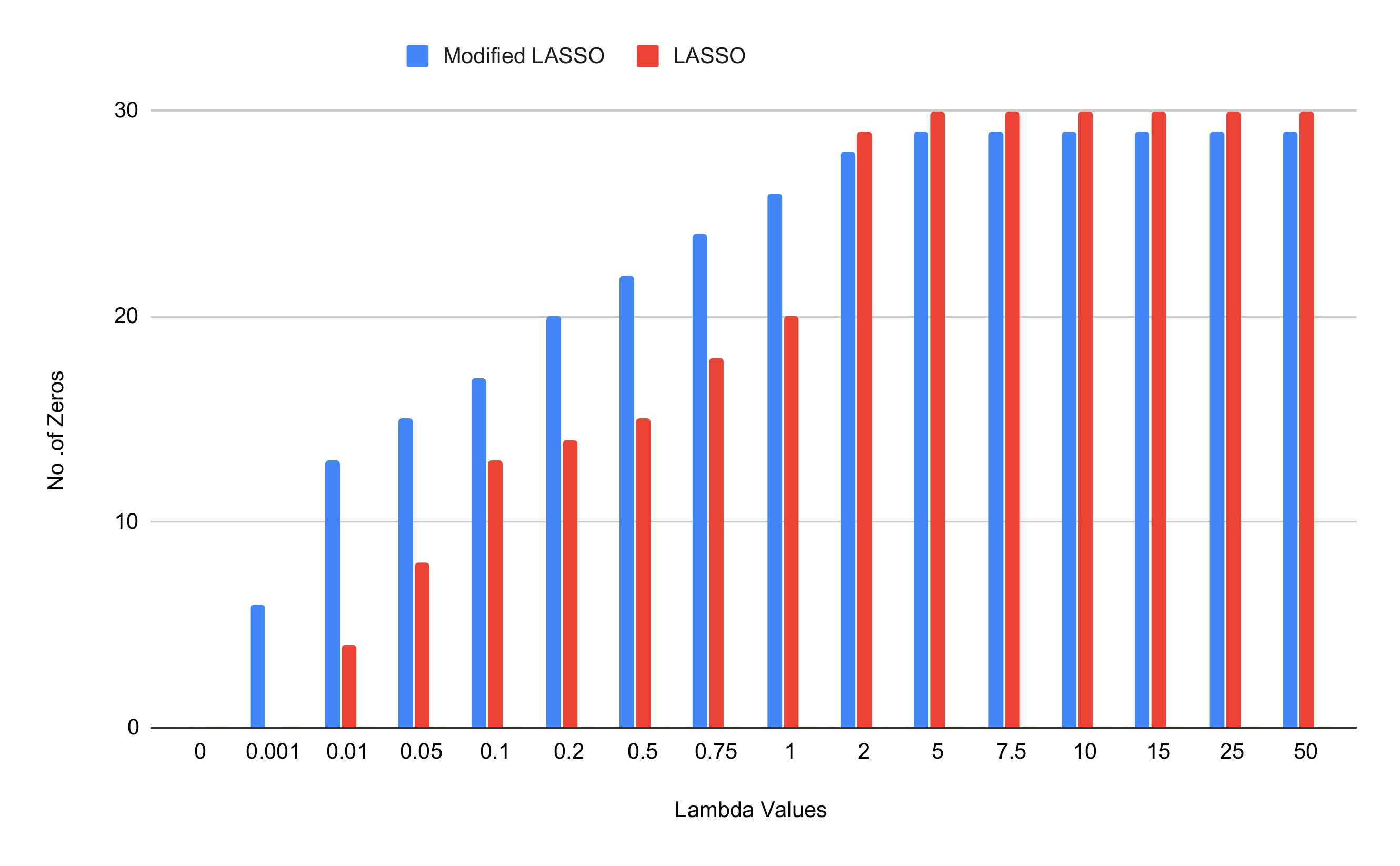}}
\caption{Sparsity Induced by Lasso and Modified Lasso}
\label{ModifiedLASSOsparse}
\end{center}
\vskip -0.2in
\end{figure}
In the next experiment we compared the performance of a subsample of data sampled using ridge leverage scores with the subsample sampled using uniform sampling. We solved the modified lasso problem on the original data and recorded the objective function value as $V1$. The ridge leverage scores were calculated exactly for the data points. Then we sampled a subsample of points using ridge leverage scores from $[\M{A} ~~ \M{b}]$. The sampled points were rescaled by the reciprocal of sample size times the probability of sampling the point. We solved the modified lasso on the smaller data and recorded the solution vector. This solution vector was plugged in the modified lasso function with the original large matrix $[\M{A} ~~ \M{b}]$ and the objective function value was noted as $V2$. We calculated the relative error as $\frac{|V1-V2|}{V1}$. Similar experiment was performed using uniform sampling. In Table \ref{Table 1} we report the values of relative error for different sample sizes. Here $\lambda=0.5$. The values reported are medians of 5 random experiments for each sample size. It can be seen that even for very small sample sizes, samples obtained using ridge leverage score performed much better than uniform sampling. In fact, uniform sampling showed significant improvement in its own relative error only at much larger sample sizes of the order of $1000$'s. For e.g. at size $1000$ error was $0.6816$ and at size $2500$ error was $0.6554$.
\begin{table}[t]
\caption{Relative error of different coreset sizes for Modified Lasso, $\lambda=0.5$}
\label{Table 1}
\vskip 0.15in
\begin{center}
\begin{small}
\begin{sc}
\begin{tabular}{lcccr}
\toprule
Sample Size & Ridge Leverage & Uniform Sampling \\
& Scores Sampling & \\
\midrule
\textbf{30}          & 0.059   & 0.8289     \\ 
\textbf{50}          & 0.044   & 0.8289     \\ 
\textbf{100}         & 0.031   & 0.8286     \\ 
\textbf{150}         & 0.028   & 0.8286    \\ 
\textbf{200}         & 0.013   & 0.8287     \\ 
\bottomrule
\end{tabular}
\end{sc}
\end{small}
\end{center}
\vskip -0.1in
\end{table}
In the next experiment we fixed a sample size of $200$ and solved modified lasso for various values of $\lambda$. We report the relative error for both ridge leverage scores based sampling and uniform sampling. We report the medians of 5 random experiments for each value of $\lambda$ for both schemes of sampling in Table \ref{Table 2}. Although both the objective function and the ridge leverage scores depend on the value of $\lambda$, still we can observe a decrease in value of relative error for the same sample size with increasing $\lambda$ except for one case. This effect is more pronounced in case of uniform sampling where the probabilities are independent of $\lambda$.
\begin{table}[t]
\caption{Relative error of different $\lambda$ values, $($sample size $=200)$ for Modified Lasso}
\label{Table 2}
\vskip 0.15in
\begin{center}
\begin{small}
\begin{sc}
\begin{tabular}{lcccr}
\toprule
$\lambda$ & Ridge Leverage & Uniform Sampling \\
& Scores Sampling & \\ 
\midrule
\textbf{0.1}  & 0.026  & 2.975  \\
\textbf{0.5}  & 0.013  & 0.828 \\
\textbf{0.75} & 0.017  & 0.576 \\
\textbf{1}    & 0.014  & 0.443  \\
\textbf{5}    & 0.007  & 0.103 \\
\bottomrule
\end{tabular}
\end{sc}
\end{small}
\end{center}
\vskip -0.1in
\end{table}
In our final experiment we compared the performance of sensitivity based sampling with uniform sampling for the RLAD in similar manner as for modified lasso. Here we used the sensitivity upper bounds we calculated for the RLAD problem. The median values of relative error of 5 experiments for each sample size is reported in Table \ref{Table 3}. Sensitivity based sampling clearly outperforms uniform sampling in this case.

\begin{table}[t]
\caption{Relative error of different coreset sizes for RLAD, $\lambda=0.5$}
\label{Table 3}
\vskip 0.15in
\begin{center}
\begin{small}
\begin{sc}
\begin{tabular}{lcccr}
\toprule
Sample Size & Sensitivity  & Uniform Sampling \\
& based Sampling & \\
\midrule
\textbf{30}  & 0.69 & 385.99\\
\textbf{50}  & 0.65 & 112.70 \\
\textbf{100} & 0.34 & 98.53 \\
\textbf{150} & 0.19 & 96.09  \\
\textbf{200} & 0.17 & 27.49\\
\bottomrule
\end{tabular}
\end{sc}
\end{small}
\end{center}
\vskip -0.1in
\end{table}

\subsection{Experiments on Real Data}
\begin{table}[h!]
\caption{Relative error of different coreset sizes for Modified Lasso on Real Data, $\lambda=1$}
\label{Table 4}
\vskip 0.15in
\begin{center}
\begin{small}
\begin{sc}
\begin{tabular}{lcccr}
\toprule
Sample Size  & Uniform Sampling & Ridge Leverage \\
& & Scores Sampling & \\
\midrule
\textbf{50}          & 0.0280   & 0.0267     \\ 
\textbf{100}          & 0.0184   & 0.0161     \\ 
\textbf{150}         & 0.0119   & 0.0082     \\ 
\textbf{300}         & 0.006   & 0.0048    \\ 
\textbf{500}         & 0.0042 & 0.0028  \\ 
\bottomrule
\end{tabular}
\end{sc}
\end{small}
\end{center}
\vskip -0.1in
\end{table}
We used the \textit{Combined Cycle Power Plant Data Set} \cite{tufekci2014prediction} available at the UCI Machine learning repository. The data set has $9567$ data points and $4$ features namely \textit {Temperture(T), Ambient Pressure (AP), Relative Humidity (RH) and Exhaust Vacuum (V)}. The task is to predict the \textit{Net hourly electrical energy output (EP)}. We used the modified lasso regression model to fit this data. We normalized each feature by its maximum value so that each feature lies in the same range. We solved the  modified lasso problem on the entire data for $\lambda=1$. We also applied our ridge leverage score based sampling and uniform sampling over the data and created coresets of different sizes. We solved modified lasso problem on the coresets and used the parameter vector obtained, with the original data and compared the errors. We did 5 experiments for each coreset size and obtained the relative error values.The median of relative errors for each coreset size is recorded in Table \ref{Table 4}. As can be seen, ridge leverage score sampling performs slightly better than uniform sampling. It is expected, as real data sets tend to have uniform leverage scores. However for data with nonuniform leverage scores our method is useful in practice too.\\

To the best of our knowledge, this is the first set of experiments evaluating performance of coresets for any regularized form of regression. These clearly verify that coresets constructed using importance scores can give performance comparable to original data at much smaller sizes.
\section{Conclusion and Discussion}
In this paper we have studied the coresets for regularized regression problem. We have shown that coresets smaller than unregularized regression might not be possible for all forms of regularized regression. We have introduced the modified lasso and shown a smaller coreset for it. Understanding the statistical properties of modified lasso is an interesting independent research direction. We have also shown smaller coresets for $\ell_p$ regression with $\ell_p$ regularization by upper bounding the sensitivity scores. Finally we have shown empirical results to support our theoretical claims.

One obvious open question is to see if tighter upper bounds are possible for the sensitivity scores for these problems. A detailed empirical study of coresets for regularized versions of regression under different settings is also required as has been done for $\ell_2$ and $\ell_1$ regression problems. Coresets for regression with other type of regularization or coresets for other regularized problems is an interesting area for future work. 
\section*{Acknowledgements}
We would like to thank Jayesh Choudhari and Kshiteej Sheth for helpful discussions. We are also grateful to the anonymous reviewers for their helpful feedback. Anirban Dasgupta acknowledges the kind support of the N. Rama Rao Chair Professorship at IIT Gandhinagar, the Google India AI/ML award (2020), Google Faculty Award (2015), and CISCO University Research Grant (2016). Supratim Shit acknowledges the kind support of Additional Fellowship from IIT Gandhinagar.
\bibliographystyle{unsrt}
\bibliography{references}
\end{document}